\documentclass[11pt]{article}

\usepackage{epsfig, psfrag}
\usepackage{latexsym}
\usepackage{enumerate}
\usepackage{url}
\usepackage{float}
\usepackage[hypertexnames=false,hyperfootnotes=false]{hyperref}
\usepackage{texnansi}
\usepackage{color}
\usepackage{tikz}
\usepackage[margin=10pt,font=small,labelfont=bf]{caption}
\usepackage[subrefformat=parens,labelformat=simple]{subcaption}
\usepackage{afterpage}
\usepackage{enumitem}
\usepackage[boxed]{algorithm}
\usepackage{algpseudocode}
\usepackage[normalem]{ulem}
\usepackage{lmodern}
\usepackage{booktabs}
\usepackage{sectsty}
\usepackage{ifthen}
\usepackage{amsmath}
\usepackage{amsthm}
\usepackage{amssymb}
\usepackage{amsfonts}
\usepackage{thmtools}
\usepackage{thm-restate}
\usepackage{xspace}
\usepackage{titling}
\usepackage{parskip}
\usepackage{natbib}
\usepackage{xfrac}
\usepackage{multirow}
\usepackage{bigdelim}
\usepackage{bm}
\usepackage{makecell}
\usepackage{colortbl}
\usepackage{xcolor}
\PassOptionsToPackage{algo2e,ruled}{algorithm2e}
\RequirePackage{algorithm2e}
\newcommand{\field}[1]{\ensuremath{\mathbb{#1}}}
\newcommand{\R}{\ensuremath{\field{R}}} 
\newcommand{\I}[1]{\ensuremath{\mathbb{I}_{\left\{#1\right\}}}} 
\newcommand{\PR}{\ensuremath{\mathbb{P}}} 
\newcommand{\E}{\ensuremath{\mathbb{E}}} 
\newcommand{\defeq}{\ensuremath{\triangleq}}

\newcommand{\Ascr}{\ensuremath{\mathcal A}}

\newcommand{\Escr}{\ensuremath{\mathcal E}}

\newcommand{\Gscr}{\ensuremath{\mathcal G}}
\newcommand{\Hscr}{\ensuremath{\mathcal H}}
\newcommand{\Iscr}{\ensuremath{\mathcal I}}

\newcommand{\Nscr}{\ensuremath{\mathcal N}}

\newcommand{\Pscr}{\ensuremath{\mathcal P}}

\newcommand{\Rscr}{\ensuremath{\mathcal R}}

\newcommand{\Yscr}{\ensuremath{\mathcal Y}}

\DeclareMathOperator{\Cov}{Cov}

\DeclareMathOperator*{\argmax}{\mathrm{argmax}}

\declaretheoremstyle[headfont=\sffamily\bfseries,bodyfont=\itshape]{thm-sf}
\declaretheorem[style=thm-sf]{theorem}

\declaretheorem[style=thm-sf]{definition}
\declaretheorem[style=thm-sf]{example}

\declaretheorem[style=thm-sf]{corollary}

\renewcommand{\thmcontinues}[1]{\hyperref[#1]{continued}}
\usetikzlibrary{arrows,patterns,plotmarks,pgfplots.groupplots}
\tikzstyle{every picture} += [>=stealth]
\tikzset{axis/.style={semithick, line join=miter}}
\allsectionsfont{\sffamily}
\makeatletter
\def\@seccntformat#1{\csname the#1\endcsname.\quad}
\makeatother

\floatname{algorithm}{\normalfont\sffamily\bfseries Algorithm}
\floatstyle{ruled}
\newcommand{\emailhref}[1]{\href{mailto:#1}{\tt #1}} 
\provideboolean{fastcompile}
\newcommand{\hidefastcompile}[1]{\ifthenelse{\boolean{fastcompile}}{}{#1}}
\newcommand{\todo}[1]{{\color{red} \noindent {\sffamily\bfseries TODO:} #1}}

\usepackage{pgfplots}
\usepackage{pgfplotstable}
\usetikzlibrary{calc}
\usepackage{mathtools}
\definecolor{orange}{rgb}{0.85,0.33,0.13} 
\definecolor{green}{rgb}{0.13,0.85,0.33}
\definecolor{purple}{rgb}{0.33,0.13,0.85}
\definecolor{lime}{rgb}{0.65,0.85,0.13}
\definecolor{blue}{rgb}{0.13,0.65,0.85}
\pgfplotscreateplotcyclelist{tricolor}{%
  orange,every mark/.append style={fill=orange!80!black},mark=*\\%
  green,every mark/.append style={fill=green!80!black},mark=square*\\%
  purple,every mark/.append style={fill=purple!80!black},mark=otimes*\\%
  black,mark=star\\%
  orange,every mark/.append style={fill=orange!80!black},mark=diamond*\\%
  green,densely dashed,every mark/.append style={solid,fill=green!80!black},mark=*\\%
  purple,densely dashed,every mark/.append style={solid,fill=purple!80!black},mark=square*\\%
  black,densely dashed,every mark/.append style={solid,fill=gray},mark=otimes*\\%
  orange,densely dashed,mark=star,every mark/.append style=solid\\%
  green,densely dashed,every mark/.append style={solid,fill=green!80!black},mark=diamond*\\%
}
\pgfplotsset{colormap={tricolormap}{color=(orange) color=(green) color=(purple)},
  colormap={quadcolormap}{color=(orange) color=(lime) color=(blue) color=(purple)}}
\pgfplotstableset{%
  font=\small,
  every head row/.style={before row=\toprule[1pt], after row=\midrule},
  every last row/.style={after row=\bottomrule[1pt]}}

\renewcommand*{\thead}[1]{\bfseries \makecell{#1}}

\usepackage{setspace}
\usepackage[margin=1in]{geometry}

\setcitestyle{authoryear,round,semicolon,aysep={,},yysep={,},notesep={, }}
\provideboolean{submissionversion}

\ifthenelse{\boolean{submissionversion}}{
  \renewcommand{\todo}[1]{}
  
  \newcommand{\deledit}[1]{}
}{
  
  \newcommand{\deledit}[1]{{\color{orange} \sout{#1}}}
}

\tikzstyle{rate} += [color=orange,very thick]
\usetikzlibrary{matrix,positioning}
\usepgfplotslibrary{groupplots}
\pgfplotsset{compat=newest}

\newcommand{\seclabel}[1]{\noindent\textbf{\sffamily #1}}

\ifthenelse{\boolean{submissionversion}}{
  \title{\textsf{\textbf{Policy Gradient Optimization of\\ Thompson Sampling Policies}}}
  \author{}
  \date{}
}{
  \title{\textsf{\textbf{Policy Gradient Optimization of\\ Thompson Sampling Policies}}%
    \thanks{The second and third authors wish to thank Paolo Baudissone for early
      exploration of some the ideas in this paper.}
  }
\author{ \\
  Seungki Min \\
  Graduate School of Business \\
  Columbia University \\
  \emailhref{smin20@gsb.columbia.edu} 
  \and \\
  Ciamac C. Moallemi \\
  Graduate School of Business \\
  Columbia University \\
  \emailhref{ciamac@gsb.columbia.edu}  \\
  \and \\
  Daniel J. Russo \\
  Graduate School of Business \\
  Columbia University \\
 \emailhref{djr2174@gsb.columbia.edu} \\
}

\date{
 Current Revision: June 2020}
}

\begin{document} 

\maketitle
\singlespacing

\begin{abstract}
  We study the use of policy gradient algorithms to optimize over a class of generalized Thompson
  sampling policies. Our central insight is to view the posterior parameter sampled by Thompson
  sampling as a kind of pseudo-action. Policy gradient methods can then be tractably applied to
  search over a class of sampling policies, which determine a probability distribution over
  pseudo-actions (i.e., sampled parameters) as a function of observed data. We  also propose 
  and compare policy gradient estimators that are specialized to Bayesian bandit problems. 
  Numerical experiments demonstrate that direct policy search on top of Thompson sampling automatically corrects for
  some of the algorithm's known shortcomings and offers meaningful improvements even in long
  horizon problems where standard Thompson sampling is extremely effective.
\end{abstract}


\onehalfspacing


\section{Introduction} \label{sec:intro}

In both academia and industry, Thompson sampling has emerged as a leading approach to exploration
in online decision making. This is driven by the algorithm's simplicity, generality, ability to
leverage rich prior information about problem, and its resilience to delayed feedback. But,
like most popular bandit algorithms, it is a heuristic design based on intuitive appeal and some
degree of mathematical insight. The tutorial paper by \citet{russo2018tutorial} details numerous
settings in which Thompson sampling can be grossly suboptimal. We highlight several such
situations:
\begin{itemize}
\item Settings where the time horizon is short relative to the number of arms. As an extreme case,
  in the situation there is a single period remaining, the myopic policy is optimal and Thompson
  sampling will over explore. At another extreme,  if there are many
  arms, it may be optimal to only restrict exploration to a subset so that a good arm can be
  identified in time to exploit over a reasonable time frame.
\item Thompson sampling does not directly consider reward noise. If there is a significant
  heterogeneity in the noise across arms, Thompson sampling may suboptimally pull noisy arms about
  which there is little hope to learn.
\item In settings with correlated arms, pulling a single arm may provide information about many
  other arms. In these settings, there may be ``free exploration'' where, for example, a myopic
  policy might learn about all arms and the type of explicit exploration undertaken by Thompson
  sampling may be wasteful.
\end{itemize}
An underlying theme in the above example is the fact that Thompson sampling does not make an
explicit exploration-exploitation trade off. Even in less extreme settings, Thompson sampling is
generally thought to explore too aggressively.

Thompson sampling is designed for a Bayesian multi-armed bandit problem, 
a well-defined optimization problem that has long been approached using the
tools of dynamic programming. Despite this, the literature offers no way to use computation,
rather than human ingenuity, to improve on standard Thompson sampling. We propose and benchmark
the use of policy gradient methods to optimize over a given family of Thompson sampling style
algorithms. The proposed methods use substantial offline computation, but the resulting policies
can be executed without additional real-time computation.

At first glance, it appears standard policy gradient algorithms \citep{williams1992simple} cannot
be efficiently applied to Thompson sampling. The challenge is that traditional policy gradient
methods require computation of the score function of the distribution of actions. While, in
principle, Thompson sampling randomly draws an action at each decision point, the distribution
over actions from which it samples is not available in closed form. Instead, efficient
implementations sample a model parameter from a posterior distribution and then select the action
that is optimal under this sampled parameter. Under such an implementation, it may be difficult to compute the probability of selecting each action.

Our central insight is as follows: we view the posterior parameter sampled by Thompson sampling as
a kind of ``pseudo-action''. In our framework, a \emph{sampling policy} maps the history of
observations to a probability distribution over the parameter space. A full algorithm will, in
each period, draw a sample from according to the sampling policy and subsequently apply the base
action that is optimal under the sampled parameter. In standard Thompson sampling, the sampling
policy applies Bayes' rule, mapping any history to the associated posterior distribution over
pseudo-actions (parameters). Mathematically, this is equivalent to the standard formulation which
views the decision as a choice of base action. Critically, however, by viewing the decision as a
choice of pseudo-action (parameter), the distribution of pseudo-actions for Thompson sampling is
often available in closed-form: it is simply the posterior distribution of the parameter.  This simple
but powerful shift in perspective enables us to search over Thompson sampling style policies using
policy gradient methods.  Indeed, we will use policy gradient to search over a class of sampling
policies that are themselves parameterized by hyper-parameters we call \emph{meta-parameters}.

A sampling policy could be parameterized in many ways. One option is to parameterize them by
complex neural networks. Our experiments demonstrate that even simple modifications of standard
Thompson sampling offer substantial benefit. One approach builds on Thompson sampling by viewing
the statistical parameters of the Bayesian model (e.g., the prior distribution, the noise
distribution, etc.) as meta-parameters. Another takes the posterior distribution used by standard
Thompson sampling and reshapes it.  Policy gradient methods for searching over the meta-parameters
are tractable as long as (i) the sampling policies can be applied efficiently and (ii) given any
history, one can efficiently calculate derivatives of the sampling distribution's log density with the respect to the meta-parameters. Typically (ii) requires that probability
density function is known up to a proportionality constant.

Our work has a close conceptual connection to work on meta-learning \citep[see
e.g.,][]{finn2017model, baxter2000model}.  As is nicely articulated by \citet{bastani2019meta},
many companies face a large sequence of experimentation tasks, raising the question of how to
effectively share information across these tasks. Consider a web company who may run thousands of A/B
tests per year, giving them strong prior knowledge of the distribution of effect sizes and click
through rates. Or a news article recommendation service has a new set of articles each day and
needs to experiment to learn which will be popular. Each day can be viewed as its own instance of
a bandit problem and the platform's goal is to do well on average across a large number of
days. \citet{bastani2019meta} suggest an empirical Bayesian approach, where the prior of Thompson
sampling is statistically estimated from data on previous tasks. This view of meta-learning as
learning a prior distribution has long been recognized. Our approach, however, will not be to
apply Thompson sampling directly using some form of statistically learned prior, since Thompson
sampling is not itself an optimal policy. If historical data can be used to build a simulator of
this meta-bandit problem, then it is more appropriate to search over Thompson sampling like policies
aiming to directly optimize the true performance metric, the average reward. This idea ---
shifting from learning elements of the statistical model such as the prior distribution or noise
model via statistical estimation to direct optimization --- may especially be powerful in settings
where the statistical model is mis-specified.

Our contributions are as follows:
\begin{enumerate}
\item \textbf{\sffamily We develop a tractable framework for policy gradient estimation for
    sampling policies.}

  Several recent works have explored the use of gradient based search to tune bandit algorithms
  \citep{duan2016rl, boutilier2020differentiable}. Relative to these works, one of our main
  contributions is to uncover a way to apply policy gradient methods to Thompson sampling,
  allowing us to fine-tune a widely used algorithm with strong theoretical guarantees.  Very
  recently, an independent and contemporaneous pre-print by \citet{kveton2020differentiable}
  discovered a similar approach to tuning Thompson sampling.

\item \textbf{\sffamily We provide and analyze multiple gradient estimators for sampling policies.}

  As in the broader application of policy gradient for reinforcement learning, there are
  multiple possible gradient estimators possible, through different choices of reward metrics
  and baselines. We derive several novel policy gradient estimators that are specifically
  tailored to Bayesian bandit problems. We are able to compare their variance theoretically and
  empirically.

\item \textbf{\sffamily We computationally demonstrate the benefits of our approach.}

  Through simple numerical experiments, we provide a compelling proof of concept. Policy search
  produces policies that correct for shortcomings of Thompson sampling in short horizon problems
  or problems with large discrepancies between the variances of arm rewards. Perhaps more
  surprisingly, policy search offers substantial improvements over Thompson sampling even
  in a canonical long horizon problem to which it is ideally suited. We also compare against
  optimistic Gittins indices \citep{gutin2016optimistic}, information directed sampling
  \citep{russo2018learning}, and Bayesian upper confidence bound algorithms
  \citep{kaufmann2012bayesian}, confirming that direct policy search on top of Thompson sampling
  produces state of the art results for widely studied problem settings. In the future, we hope
  to extend the numerical experiments beyond problems with independent arms.
\end{enumerate}

%


\section{Model} \label{sec:model}

We consider a multi-armed bandit problem in a Bayesian setting.

\seclabel{Rewards.}
Let $\Ascr$ be the set of arms, possibly infinite, among which the decision maker (DM) can select at each time $t = 1, \ldots, T$.
When the DM pulls an arm $a \in \Ascr$ at time $t$, they earn a random reward $R_{a,t}$ drawn from a distribution $\Rscr$ that is parameterized by \emph{model parameters} $\theta \in \Theta$, i.e.,
\begin{equation}\label{eq:r-dist}
        R_{a,t} | \theta \, \sim \, \Rscr( \theta, a )  .
\end{equation}
We assume that the rewards $(R_{a,t})_{t \in [T]}$ are conditionally independent and identically
distributed\footnote{One important case where the rewards are \emph{not} i.i.d.\ is the case of
  contextual bandits. Here, the reward at time $t$ is given by
  $R_{a,t} | \theta,x_t \, \sim \, \Rscr( \theta, a, x_t )$. Here $x_t$ is the context at time
  $t$, and this is a stochastic process that evolves independently of the DM's actions. The
  framework we develop here can be extended to accommodate contextual cases as well in a
  straightforward fashion, but in the interest of simplifying the presentation we will not be
  explicit about this.}  given $\theta$, for each $a \in \Ascr$.  We further define
$\mu \colon \Theta \times \Ascr \rightarrow \mathbb{R}$ to be the \emph{mean reward function},
i.e.,
\begin{equation} \label{eq:mean-reward}
        \mu( \theta, a ) \defeq \E_{R \sim \Rscr(\theta,a)}\left[ R \right]   .
\end{equation}
As we consider a Bayesian setting, we view $\theta$ as a (multivariate) random variable that is
sampled from a \emph{prior distribution}, which we denote by $\Pscr( y_0 )$, i.e.,
\begin{equation}
        \theta \sim \Pscr( y_0 )  ,
\end{equation}
where $y_0 \in \Yscr$ are the sufficient statistics that parameterize the prior distribution.

\seclabel{Information set.}
The parameters $\theta$ are unknown to the DM, but can be inferred by observing the reward realizations sequentially revealed over time.
More precisely, let $H_t$ be the history, or information revealed up to time $t$ (inclusive):
\begin{equation} \label{eq:history}
        H_t \defeq \left( A_s, R_{A_s,s} \right)_{s \in [t]}.
\end{equation}
This includes the actions taken by the DM and the rewards realized through time $t$.  We assume that
the DM has knowledge of both the prior distribution $\Pscr(y_0)$ and the functional form of reward
distribution $\Rscr$.

As a Bayesian learner, the DM will update their belief according to Bayes' rule whenever observing a new reward realization, and thus maintain a \emph{posterior distribution} for $\theta$ at each time.
Without loss of generality,\footnote{
                When $\Pscr$ is a conjugate prior of $\Rscr$ and belongs to the exponential family, the sufficient statistics $Y_t$ will admit a compact representation.
                In other cases, $Y_t$ may represent the entire history, i.e., $Y_t = H_t$.
} we assume that the posterior is represented as
\begin{equation} \label{eq:posterior}
        \theta | H_t ~\sim~ \Pscr(Y_t).
\end{equation}
Here, $Y_t \in \Yscr$ is a random variable that denotes the sufficient statistics of the posterior distribution after observing the history $H_t$ (i.e., after observing the $t^\text{th}$ reward realization).
We set $Y_0 = y_0$.

We will describe the randomness of our stochastic model more explicitly as follows.  The DM's policy $\pi$
is described by a sequence of deterministic mappings $(\pi_t)_{t \in [T]}$. Each mapping
$\pi_t \colon \Hscr_{t-1} \times \Escr \rightarrow \Ascr$ specifies the next action $A_t$ as a
function of history $H_{t-1} \in \Hscr_{t-1}$ that is revealed immediately prior to time $t$, and
random noise $\epsilon_t \in \Escr$ that can be utilized for randomization in the choice of
action.  Similarly, the reward realization is described by a mapping $r\colon\Theta \times \Ascr \times \Xi \rightarrow \R$ that specifies the next
reward realization, where any randomness is generated by the noise variable $\xi_t \in \Xi$. That
is,
\begin{equation} \label{eq:explicit-randomness}
        A_t = \pi_t( H_{t-1}, \epsilon_t )  , \quad
        R_{a,t} = r( \theta, a, \xi_t )  .
\end{equation}
We assume, without loss of generality, that the noise random variables
$(\epsilon_t)_{t \in [T]}$ are independent and identically distributed, as are
$(\xi_t)_{t \in [T]}$.

We define an \emph{instance} or \emph{episode}, denoted by $I$, as a random variable that encodes
all uncertainties in the environment, but not the randomness in the DM's decision rule:
\begin{equation}\label{eq: instance}
        I \defeq \left( \theta, ( \xi_t )_{t \in [T]} \right)  .
\end{equation}
In other words, given an instance $I$, we exactly know what rewards will be realized for any given
action sequence $a_{1:T} \in \Ascr^T$ committed by the DM.  The set of all possible instances is
denoted by $\Iscr$.

\seclabel{Objective.}  The DM aims to earn as much reward as possible in expectation.  Given the
DM's decision rule $\pi$, its \emph{expected total reward}, denoted by $\textsc{Reward}[\pi]$, is
defined as
\begin{equation}
        \textsc{Reward}[\pi] \defeq \E\left[ \sum_{t=1}^T R_{A_t,t} \right],
\end{equation}
where the expectation is taken with respect to the randomness of instance (i.e., the randomness of
the parameters $\theta$ and the reward realizations) and also any randomness of the choice of
actions (if $\pi$ is a randomized policy).


To better illustrate our setup, we provide an example of a canonical multi-armed bandit problem
described with the notation introduced above.

\begin{example}[Gaussian MAB] \label{ex:gaussian-MAB} Consider a finite number of arms
  $\Ascr \defeq \{1, \ldots, K\}$. Each arm $a$ yields normally distributed rewards with an
  unknown mean $\theta_a$ and a known variance $\sigma_a^2$, where the prior of $\theta_a$ is
  given by a normal distribution $\Nscr( m_{a,0}, v_{a,0}^2)$. The model parameters are given by
  the vector $\theta \defeq (\theta_a)_{a \in [K]}$.

  Each instance takes the form $I \defeq \left( \theta, ( \xi_t )_{t \in [T]} \right)$, where
  $(\xi_t )_{t \in [T]}$ are i.i.d.\ standard normal random variables that randomize the reward
  realizations according to
        \begin{equation}
                \mu( \theta, a ) = \theta_a
                 , \quad
                r( \theta, a, \xi_t ) = \theta_a + \sigma_a \xi_t
                 .
        \end{equation}
        The posterior for a parameter $\theta_a$ after time $t$ is given by a normal distribution
        $\Nscr( m_{a,t}, v_{a,t}^2 )$. Here, the sufficient statistics $m_{a,t}$ and $v_{a,t}^2$
        can be computed in a closed-form according to
        \[
        \begin{split}
          m_{a,t} = v_{a,t}^{-2} \times \left( v_{a,0}^{-2} \cdot m_{a,0} + \sigma_a^{-2} \sum_{s=1}^t \I{ A_s = a } R_{A_s,s} \right)
          , \quad
          v_{a,t}^2 = \left( v_{a,0}^{-2} + \sigma_a^{-2} \sum_{s=1}^t \I{ A_s = a} \right)^{-1}
          .
        \end{split}
      \]
      As a collection of these sufficient statistics across the arms,
      $Y_t \defeq ( m_{a,t}, v_{a,t}^2 )_{a \in [K]} \in \mathbb{R}^{2K}$, determine the posterior
      of the parameters $\theta$ given the history $H_t$.
\end{example}


\section{Parameterized Thompson Sampling} \label{sec:param-ts}

Thompson sampling (\textsc{TS}) \citep{thompson1933likelihood} is a randomized policy that works as
follows.  At each time $t = 1, \ldots, T$: (i) the parameters $\tilde{\theta}_t$ are sampled from
the posterior distribution $\Pscr( Y_{t-1} )$ given all information prior to time $t$; and (ii) an
action is chosen to maximize the expected reward given these sampled parameters
$\tilde{\theta}_t$. In other words,
\begin{equation}\label{eq:TS}
        \tilde{\theta}_t \sim \mathcal{P}( Y_{t-1} ) ,
        \quad
        A_t^{\textsc{TS}} \gets \argmax_{a \in \Ascr} \mu( \tilde{\theta}_t, a ) ,
\end{equation}
where $Y_{t-1} \in \Yscr$ are, as introduced in \eqref{eq:posterior}, the sufficient statistics
describing the posterior distribution of true parameters $\theta$ given the history $H_{t-1}$.  As
time progresses, the above procedure is repeated, while updating the posterior distribution
according to Bayes' rule.

An important characteristic of \textsc{TS} is ``probability matching''. Under
\textsc{TS},  the probability that an arm $a$ is
selected at time $t$ is equal to the probability that the arm $a$ is indeed the best one that
Bayesian inference predicts, i.e.,
\begin{equation} \label{eq:arm-selection-prob}
  \PR\big( A_t^\textsc{TS} = a \, | \, H_{t-1} \big)
  = \PR\left( a = \argmax_{a' \in \Ascr} \mu( \theta, a' ) \, | \, H_{t-1} \right) ,
  \quad \forall a \in \Ascr.
\end{equation}
However, probability in \eqref{eq:arm-selection-prob} is difficult to evaluate since it does not
admit a closed-form expression in most cases and does not admit feasible policy gradient
estimators.

We consider a class of variants of \textsc{TS} where the sampling policy in \eqref{eq:TS} is not
the posterior distribution, but instead is some other distribution parameterized by
\emph{meta-parameters} $\lambda \in \Lambda \subseteq \mathbb{R}^d$.  In other words, given
$\lambda$, the corresponding \emph{sampling policy} $\textsc{TS}(\lambda)$ repeats the following
at each time $t$:
\begin{equation} \label{eq:general-parameterization}
        \tilde{\theta}_t \sim \Pscr_\lambda( H_{t-1} ),
        \quad
        A_t^{\textsc{TS}(\lambda)} \gets \argmax_{a \in \Ascr} \mu( \tilde{\theta}_t, a ),
\end{equation}
where $\Pscr_\lambda( H_{t-1} )$ is a distribution on the parameter space $\Theta$ that has
arbitrary dependency on the meta-parameters $\lambda$ and the history $H_{t-1}$.  The sampling
policy $\textsc{TS}(\lambda)$ is almost identical to the na\"ive \textsc{TS} except that it
samples the parameters from $\Pscr_\lambda(H_{t-1})$ instead of $\Pscr(Y_{t-1})$. In this way,
sampling policies can be viewed as a natural generalization of $\textsc{TS}$, emitting at each
time a randomized pseudo-action $\tilde \theta_t$ (choice of parameters) from which a base action
is determined, rather than directly emitting a base action.

As we would like to employ policy gradient methods to optimize over the meta-parameters $\lambda$,
we will assume that the probability density (or mass) function of the distribution
$\Pscr_\lambda( H_{t-1} )$ is differentiable with respect to $\lambda$ over its domain $\Lambda$,
for any realization of history $H_{t-1}$ almost surely. Aside from this, the distribution
$\Pscr_\lambda( H_{t-1} )$ defining the sampling policy is allowed to be essentially
arbitrary. However, in order to illustrate our ideas, consider the following example:

\begin{example}[Posterior reshaping.]
We adopt and generalize the idea proposed in \citet{ChapelleLi}, and let the algorithm to sample
the parameters from the a reshaped posterior distribution,
\begin{equation} \label{eq:posterior-reshaping}
  \Pscr_\lambda( H_{t-1} ) = \Pscr\big( \phi_\lambda( Y_{t-1} ) \big) ,
\end{equation}
where $ \Pscr( \cdot)$ is the posterior distribution defined in \eqref{eq:posterior}, and
$\phi_\lambda \colon \Yscr \rightarrow \Yscr$ is a differentiable mapping, parameterized by
$\lambda$, that transforms one set of sufficient statistics to another.
\end{example}

Posterior reshaping is motivated by several arguments. Compared to the general parameterization
\eqref{eq:general-parameterization}, the posterior reshaping does not parameterize ``how to
learn''.  Instead, it parameterizes how to utilize the learned results, while maintaining the
Bayesian learning logic as it is.  This can significantly reduce the effort required for tuning
the meta-parameters. Moreover, its implementation requires a minimal effort once one has
already implemented the standard \textsc{TS}. Indeed, when $\phi_\lambda(\cdot)$ is the identity
map, posterior reshaping reduces to standard \textsc{TS}. Hence, under appropriate technical
assumptions, a local policy gradient search starting at the identity map will be guaranteed to do
no worse than standard \textsc{TS}.

As discussed in the introduction, the standard \textsc{TS} suffers from the over-exploration,
for example when the time horizon is short relative to the number of arms. Posterior reshaping can
naturally address this, by reducing uncertainty in the sampling distribution.  As an extreme
example, consider a situation where we are given a single time period (i.e., $T=1$).  The optimal
policy is myopic --- the optimal action is to pick the arm with the largest prior mean --- which
can be implemented by reshaping the posterior distribution to concentrate on the prior mean. Such
posterior concentration also appears in the work of \citet{information-relaxation-sampling}. The
\textsc{Irs.FH} policy they suggest is posterior reshaping sampling policy.

Furthermore, it is possible that \textsc{TS} with the correct model parameters (e.g., the
specification of the prior distribution or the reward distribution) may not be optimal for the
performance of the algorithm.  Within the framework of posterior reshaping, we can find the set of
model parameters that are empirically tuned for performance so as to outperform to the one with
the correct values.\footnote{Even if the model is mis-specified, the policy gradient method can be
  applied, but we need to be careful in the choice of gradient estimator in order to avoid a bias
  in the gradient estimation.  See the related discussion in
  Section~\ref{sec:reward-metric-choices}.  }

A concrete examples of posterior reshaping in the Gaussian case can be developed as follows:
\begin{example}[Posterior reshaping for Gaussian MAB] \label{ex:posterior-reshaping-Gaussian}
  Using the notation of Example~\ref{ex:gaussian-MAB}, standard \textsc{TS} in Gaussian MAB
  samples the parameters $\tilde{\theta}_t = ( \tilde{\theta}_{a,t} )_{a \in \Ascr}$ according to
  \begin{equation}
    \tilde{\theta}_{a,t} \sim \Nscr\left(
      \frac{ m_{a,0} + \frac{v_{a,0}^2}{ \sigma_a^2} \cdot S_{a,t-1} }{ 1 + \frac{v_{a,0}^2}{ \sigma_a^2}  \cdot N_{a,t-1}  }, ~
      \frac{ v_{a,0}^2 }{1 + \frac{v_{a,0}^2}{ \sigma_a^2}  \cdot N_{a,t-1}}
    \right),
  \end{equation}
  for each $a=1,\ldots,K$, where the sufficient statistics are given by
  $S_{a,t-1} \defeq \sum_{s=1}^{t-1} \I{ A_s = a } R_{A_s,s}$, and
  $N_{a,t-1} \defeq \sum_{s=1}^{t-1} \I{ A_s = a}$.

  We consider posterior reshaping with meta-parameters
  $\lambda \defeq \left( \lambda_a^m, \lambda_a^v, \lambda_a^\sigma, \lambda_a^\gamma \right)_{a
    \in \Ascr} \in \mathbb{R}^{4K}$ under which $\tilde{\theta}_{a,t}$ is sampled from
  \begin{equation}
    \tilde{\theta}_{a,t} \sim \Nscr\left(
      \frac{ \lambda_a^m + \lambda_a^\sigma \cdot S_{a,t-1} }{ 1 + \lambda_a^\sigma \cdot N_{a,t-1}  }, ~
      \frac{ \lambda_a^v ( 1- t/T)^{\lambda_a^\gamma} }{ 1 + \lambda_a^\sigma \cdot N_{a,t-1}}
    \right).
  \end{equation}
  This policy reduces to standard \textsc{TS} if we take $\lambda_a^m = m_{a,0}$ (prior mean),
  $\lambda_a^v = v_{a,0}^2$ (prior variance), $\lambda_a^\sigma = v_{a,0}^2 / \sigma_a^2$
  (precision ratio), and $\lambda_a^\gamma = 0$ (variance decay exponent).  The amount of
  exploration is controlled by $\lambda_a^v$ and $\lambda_a^\gamma$. In particular, the term
  $( 1 - t/T )^{\lambda_a^\gamma}$ diminishes exploration near the end of the horizon, where
  the benefit from exploration is limited.

  Note that this parameterization scheme can be represented in the form of
  \eqref{eq:posterior-reshaping}, since the sufficient statistics of the realized observations
  (i.e., $S_{a,t-1}$ and $N_{a,t-1}$) are uniquely determined from those of current posterior
  distribution (i.e., $m_{a,t-1}$ and $v_{a,t-1}^2$) and prior distribution (i.e., $m_{a,0}$ and
  $v_{a,0}^2$).  Also note that the probability density function is differentiable with respect to
  $\lambda$ given that $\lambda_a^v > 0$ and $\lambda_a^\sigma > 0$.
\end{example}

A more complex example is as follows:
\begin{example}[Deep recurrent neural network parameterization]\label{ex:dnn}
  One might consider a recurrent neural network (RNN) structure with, at each time $t$, input
  $(A_t, R_{A_t,t})$, hidden state $\tilde Y_t$ and output being the sampled pseudo-action
  $\tilde \theta_t$. The network would evolve according to
  \[
    \tilde Y_{t} \gets \phi^Y_{\lambda_Y}(\tilde Y_{t-1}, A_t, R_{A_t,t}),
    \quad
    \tilde \theta_t \sim \mathcal{P}\left(\phi^\theta_{\lambda_\theta}(\tilde Y_{t-1})\right).
  \]
  Here, the hidden state $\tilde Y_t$ is analogous to a sufficient statistic in that it summarizes
  the history up to an including time $t$. Two deep neural networks, $\phi^Y_{\lambda_Y}(\cdot)$
  and $\phi^\theta_{\lambda_\theta}(\cdot)$, with weights $\lambda_Y$ and $\lambda_\theta$, govern
  the evolution of the hidden state $\tilde Y_t$ and the output $\tilde \theta_t$,
  respectively. The meta-parameters $\lambda\defeq (\lambda_Y,\lambda_\theta)$ would be optimized
  with policy gradient methods.
\end{example}
Example~\ref{ex:dnn} is in the spirit of the approach of \citet{duan2016rl} and
\citet{boutilier2020differentiable}, where RNNs were fit with policy search methods, but where
the policies output actions. Here, to contrast, the RNN outputs distributional parameters which
are then sampled, leveraging on top of the structure of Thompson sampling.


\section{Policy Gradient for Thompson Sampling}

We aim to search over the meta-parameters $\lambda \in \mathbb{R}^d$ so that the corresponding policy $\textsc{TS}(\lambda)$ improves over the original \textsc{TS} significantly. For this purpose, we adopt the policy gradient framework, which applies variants of stochastic gradient ascent to optimize total expected reward. Formally, one can  have in mind the iteration,
\begin{equation}\label{eq: sgd}
\lambda_{k+1} = \lambda_k + \alpha_k G(\lambda_k, w_k)
\end{equation}
where $(\alpha_k)$ is a step-size sequence, $(w_{k})$ is a sequence of i.i.d.\ random
variables, and $G(\lambda_k, w_k)$ is an unbiased gradient, i.e.,
\[
  \E\left[ G(\lambda, w_k ) \right]  =
  \nabla_\lambda \, \textsc{Reward}\big[ \textsc{TS}(\lambda) \big]
= \nabla_\lambda \, \E\left[ \sum_{t=1}^T R_{A_t^{\textsc{TS}(\lambda)},t} \right].
\]
Typically, the $w_k$ denote the randomness used by a stochastic simulator. For the gradient
estimators we use, $w_k\defeq \big( \epsilon^{(k)}_t, \xi^{(k)}_t \big)_{t\in [T]}$ consists of
realizations of the random noise terms that determine the reward realizations and the action
selection of a randomized algorithm. In deriving and comparing gradient estimators, we omit the
dependence on $k$.  It is worth noting that the iteration \eqref{eq: sgd} is meant for
illustrative purposes, and other first order stochastic methods, e.g., Adam
\citep{kingma2014adam}, can also be utilized.

\subsection{Score Function Gradient Estimation}


Most implementations of policy gradient use score function gradient estimation \citep{williams1992simple}. However, the conventional scheme requires computing $\nabla_{\lambda} \log  \PR\big( A_t^{\textsc{TS}(\lambda)} = a \, \big)$, which is typically intractable since there is no closed-form expression for the distribution of the chosen action.

We circumvent this issue by interpreting the sampled parameters $\tilde{\theta}_t$ as an
(pseudo-)action taken by the policy at time $t$.  One can imagine an equivalent bandit environment
whose action space is set to the parameter space $\Theta$ and the decision maker earns the reward
$\tilde{R}_{\tilde{\theta}_t, t}\defeq R_{A_t, t}$ associated with the arm
$A_t \defeq \argmax_a \mu( \tilde{\theta}_t, a )$ as a result of his decision
$\tilde{\theta}_t$. Assume that the sampling distribution $\Pscr_\lambda( H_{t-1} )$ under the any
$H_{t-1}$ has a probability density function $p_{\lambda}( \, \cdot \, ; H_{t-1})$. Assume as well that
$p_{\lambda}( \, \cdot \, ; H_{t-1})$ is differentiable as a function of $\lambda$. This leads to the
gradient estimator
\begin{equation} \label{eq:gradient-estimator-naive}
G \defeq \sum_{t=1}^T S_t \sum_{s=t}^T R_{A_s^{\textsc{TS}(\lambda)},s}
\end{equation}
where
\begin{equation}
        S_t \defeq \nabla_{\lambda} \log p_{\lambda}\left( \tilde{\theta}_t \, ; H_{t-1} \right)
\end{equation}
denotes the score functions. This form of score function gradient estimator is well known to be
unbiased, i.e., $\E[ G ] = \nabla_\lambda \textsc{Reward}\big[ \textsc{TS}(\lambda) \big]$, which
is referred to as the policy gradient theorem in the reinforcement learning literature. Formally,
unbiasedness requires technical conditions that allow for the interchange of integrals and
derivatives. We refer to \cite{l1995note} for appropriate conditions.

\subsection{Admissible Gradient Estimators}

The standard gradient estimator \eqref{eq:gradient-estimator-naive} can be very noisy due to the
high variability of reward realizations and random action selections. In this section, we propose
a broader, more general class of gradient estimators, and demonstrate that they remain unbiased as
long as they satisfy a certain admissibility requirement. Later, we will suggest a
specific list of estimators for bandit problems, and provide a theoretical comparison among them
in terms of variance reduction.

\seclabel{General representation.}  With processes $M \defeq (M_t)_{t \in [T]}$, which we call a
\emph{reward metric}, and $B \defeq (B_t)_{t \in [T]}$, which we call a \emph{baseline}, we define
the gradient estimator $G^{M,B}$ as follows:
\begin{equation} \label{eq:gradient-estimator-general}
        G^{M,B} \defeq \sum_{t=1}^T S_t \cdot \left( M_t - B_t \right)  .
\end{equation}
The reward metric $M_t$ is a random variable that accounts for the sum of rewards that the policy
earns on the remaining horizon $t, t+1, \ldots, T$, and $B_t$ is another random variable that
represents some benchmark for the rewards over the same period.  Note that the estimator
\eqref{eq:gradient-estimator-naive} can be obtained by taking $M_t = \sum_{s=t}^T R_{A_s,s}$ and
$B_t = 0$.

\seclabel{Admissibility.}  We state a condition on the reward metric $M$ and the baseline $B$
under which they induce an unbiased gradient estimator $G^{M,B}$.

\begin{definition}[Admissible reward metric and baseline]
  A reward metric $M$ is admissible if for all $t \in [T]$ it is integrable and
  \begin{equation} \label{eq:valid-reward-metric}
    \E\left[  M_t \left| H_{t-1}, \tilde{\theta}_t \right. \right]
    = \E\left[ \left. \sum_{s=t}^T R_{A_s,s} \right| H_{t-1}, \tilde{\theta}_t \right].
  \end{equation}
  A baseline $B$ is admissible if $B_t$ is integrable and $B_t$ and $\tilde{\theta}_t$ are
  conditionally independent given $H_{t-1}$ for all $t \in [T]$, i.e.,
  \begin{equation} \label{eq:valid-baseline}
    \left. B_t \perp \!\!\! \perp  \tilde{\theta}_t \, \right| \, H_{t-1} .
  \end{equation}
\end{definition}
The first condition \eqref{eq:valid-reward-metric} ensures that a risk-neutral decision maker
would not differentiate between $M_t$ and the sum of future rewards when deciding the next
action. These two measures have the same expectation given any history (which corresponds to the
state in dynamic programming terms) and any pseudo-action $\tilde{\theta}_t$.  The second
condition \eqref{eq:valid-baseline} ensures that the decision-maker does not need to take baseline
into consideration when making a decision, since the baseline is independent of the next pseudo-action $\tilde{\theta}_t$.



The above interpretation implies that the substitution of reward metric (from
$\sum_{s=t}^T R_{A_s,s}$ to $M_t$) and the presence of baseline $B_t$ do not affect the DM's
decision at each time $t$, as long as they satisfy the admissibility conditions
\eqref{eq:valid-reward-metric}--\eqref{eq:valid-baseline}.  Therefore, we can infer that the
generalized gradient estimator \eqref{eq:gradient-estimator-general} is equal in expectation to
the standard one \eqref{eq:gradient-estimator-naive}, which is proved formally in the following
theorem.


\begin{theorem}[Unbiasedness of gradient estimator] \label{thm:unbiasedness}
       If the reward metric $M$ and the baseline $B$ are admissible, then $\E[ G^{M,B} ] = \E[G]$.
\end{theorem}
\begin{proof}
        Note that $S_t$ is measurable with respect to $\sigma\big( H_{t-1}, \tilde{\theta}_t \big)$ and
        \begin{equation}
                \E\left[ \left. S_t \right| H_{t-1} \right] = 0 ,
        \end{equation}
        due to the property of the score function.
        By the condition \eqref{eq:valid-reward-metric}, we obtain
        \[
          \begin{split}
                \E\left[ S_t M_t \right]
                 & = \E\left[ \mathbb{E}\left( S_t M_t \left| H_{t-1}, \tilde{\theta}_t
                     \right. \right) \right] \\
                 &= \E\left[ S_t \times \E\left(  M_t \left| H_{t-1}, \tilde{\theta}_t \right. \right) \right]
                 \\&= \E\left[ S_t \times \E\left( \left. \sum_{s=t}^T R_{A_s,s} \right| H_{t-1}, \tilde{\theta}_t \right) \right]
                 \\&= \E\left[ \E\left( S_t \times \left. \sum_{s=t}^T R_{A_s,s} \right| H_{t-1}, \tilde{\theta}_t \right) \right]
                 \\&= \E\left[  S_t  \times \sum_{s=t}^T R_{A_s,s} \right]
                 .
               \end{split}
             \]
             By the condition \eqref{eq:valid-baseline}, we further obtain
        \[
          \E\left[ S_t B_t \right]
          = \E\left[ \E\left(  S_t  B_t | H_{t-1} \right) \right]
          = \E\left[ \E\left(  S_t | H_{t-1}  \right) \times \E\left( B_t | H_{t-1} \right) \right]
          = 0.
        \]
        Combining these results, we deduce that
        \[
          \E\left[ G^{M,B} \right]
          = \E\left[ \sum_{t=1}^T S_t \times (M_t - B_t) \right]
          = \E\left[ \sum_{t=1}^T S_t \sum_{s=t}^T R_{A_s,s} \right]
          = \E[ G ],
        \]
        which concludes the proof.
\end{proof}

One particularly interesting class of reward metrics and baselines are those which are time separable:

\begin{example}[Time separable reward metric and baseline]\label{ex:sep}
  Consider
  \[
    M_t \defeq \sum_{s=t}^T \hat{r}_s( A_{1:s}, I ), \quad B_t \defeq b_t( A_{1:t-1}, I ),
  \]
  where $\hat{r}_t\colon \Ascr^t \times \Iscr \rightarrow \mathbb{R}$ and
  $b_t\colon \Ascr^{t-1} \times \Iscr \rightarrow \mathbb{R}$ are the deterministic functions that
  satisfy
  \[
    \E\left[ \hat{r}_t( a_{1:t}, I ) \, | \, H_{t-1}, A_{1:t-1} = a_{1:t-1} \right]
    = \E\left[ r( \theta, a_t, \xi_t ) \, | \, H_{t-1}, A_{1:t-1} = a_{1:t-1} \right]
    , \quad \forall\ a_{1:t} \in \Ascr^t,\ t \in [T].
  \]
  Then, the reward metric $M \defeq (M_t)_{t \in [T]}$ and the baseline
  $B \defeq (B_t)_{t \in [T]}$ are admissible.
\end{example}

We remark that the baseline is allowed to be \emph{instance-dependent}, meaning that it can depend
on instance $I$ defined in \eqref{eq: instance} that determines the realizations of rewards and
the true parameter $\theta$. This is a considerable generalization of the literature in which
baselines are typically chosen as a deterministic function of state
\citep{sutton2018reinforcement}. The use of common randomness \citep{glasserman1992some} in the
the baseline and the reward metric can reduce variance, especially when most variation in observed
algorithm performance is driven by different realizations of the problem instance rather than
differences in the choice of meta parameter.

\subsection{Reward Metrics and Baselines} \label{sec:reward-metric-choices}

We suggest a specific series of reward metrics and baselines that are admissible for Bayesian
bandit problems. A number of these take the time separable form of Example~\ref{ex:sep}.

\seclabel{Reward metrics.}
The followings are possible choices for reward metric:
\begin{enumerate}
        \item The observed reward $M_t^\text{obs} \defeq \sum_{s=t}^T R_{A_s,s}$.
        \item The mean reward $M_t^\text{mean} \defeq \E\left[ \left. \sum_{s=t}^T  R_{A_s,s} \right| \theta, A_{t:T} \right] = \sum_{s=t}^T \mu( \theta, A_s )$.
        \item (MAB with independent arms only) The finite-sample mean-reward estimate
          \[
            M_t^\text{fin} \defeq \sum_{s=t}^T \hat{\mu}_{I,t}(A_s),
          \] where
          $\hat{\mu}_{I,t}(a) \defeq \E\left[ \mu( \theta, a ) \left| H_T, A_s=a, \forall
              s=t,\ldots,T \right. \right] $ that indicates the best estimate for the mean reward
          of an arm $a$ that the DM can infer through a finite number of observations.\footnote{
            This metric is valid only when the arms and their associated priors are
            independent. Using the notation of \eqref{eq:explicit-randomness}, we further need to
            assume that the noise variable takes the form $\xi_t \defeq (\xi_{a,t})_{a \in \Ascr}$
            where $\xi_{a,t}$ independent across $a$.  In order for the DM to retrieve maximal
            information about a particular arm $a$, it is required to pull the arm $a$ throughout
            the entire rest of the horizon (i.e., $A_s = a, \forall s=t,\ldots,T$).  The metric
            $\hat{\mu}_{I,t}(a)$ represents the mean reward estimate that the DM will have in this
            scenario, which also has a dependency on the instance $I$.

        }
        \item The posterior mean $M_t^\text{Bayes} \defeq \sum_{s=t}^T \E\left[ \mu( \theta, A_s ) \left| H_{s-1}, A_s \right. \right]$.
        \item The state-action Q-function
          \[
            M_t^Q \defeq \E\left[ \left. \sum_{s=t}^T  R_{A_s,s} \right| H_{t-1}, A_t \right] =
            \E\left[ \left. \sum_{s=t}^T \mu( \theta, A_s ) \right| H_{t-1}, A_t \right].
          \]
\end{enumerate}
Recall that $\mu( \theta , a )$ is the mean reward function, defined in \eqref{eq:mean-reward}, that is a deterministic function representing the expected reward of an arm $a$ given the parameters $\theta$.

These metrics differ in the information set on which the conditional expectation of the sum of future rewards, $\sum_{s=t}^T R_{A_s,s}$, is taken.
The main motivation for deriving this series of metrics is ``Rao--Blackwellization,'' i.e., integrating out some of the randomness in the future reward realizations and the future action selections.
More specifically, the metric $M_t^\text{mean}$ is obtained from $M_t^\text{obs}$ by integrating out the randomness of immediate reward realization while maintaining the dependency on the (random) parameters $\theta$ and the (random) action sequence $A_{t:T}$.
The metric $M_t^\text{fin}$ is motivated from the fact that knowing the true parameters $\theta$ is as informative as having an infinite number of observations for each arm, and improves over $M_t^\text{mean}$ by taking into account how much the DM can learn about $\theta$ with a finite number observations (i.e., by integrating out the uncertainties in $\theta$ that cannot be identified).
Next, under the metric $M_t^\text{Bayes}$, the DM earns the expected reward given the posterior distribution at each time, which averages out the uncertainties in $\theta$ at each time step.
Finally, the metric $M_t^Q$ represents the Q-value of the given policy, i.e., the expected future reward of the policy at a given state (history) and an action (arm), which averages out the all uncertainties that arise after taking the action $A_t$.

We remark that these reward metrics are mostly taken from \citet{information-relaxation-sampling}.
While the reward metrics $M_t^\text{obs}$ and $M_t^Q$ are applicable for the general Markov decision processes, the other three metrics $M_t^\text{mean}$, $M_t^\text{fin}$ and $M_t^\text{Bayes}$ are valid only for bandit problems, and in particular, $M_t^\text{fin}$ and $M_t^\text{Bayes}$ are valid only in a Bayesian setting. In addition, accurate computation of $M_t^Q$ typically requires averaging over many Monte Carlo simulations (e.g., roll-outs) which may be computationally expensive. 

\seclabel{Baselines.}
We further provide a list of baselines as follows:
\begin{enumerate}
        \item The null baseline $B_t^\text{null} \defeq 0$.
        \item The oracle performance $B_t^\text{oracle} \defeq M_t^\star$ where $M_t^\star$ is the reward (measured with the corresponding reward metric) that the action sequence $A_t^\star = \argmax_{a \in \Ascr} \mu(\theta, a)$ achieves in the \emph{same instance}.
                For example, in a combination with $M_t^\text{mean}$, we obtain $B_t^\text{oracle} = \sum_{s=t}^T \max_{a \in \Ascr} \mu( \theta, a )$.
        \item The self-play baseline $B_t^\text{self} \defeq \tilde{M}_t$ where $\tilde{M}_t$ is the reward (measured with the corresponding reward metric) that an independent run of the same algorithm achieves in the same instance.
                For example, in a combination with $M_t^\text{mean}$, we obtain $B_t^\text{self} = \sum_{s=t}^T \mu( \theta, \tilde{A}_s )$ where $\tilde{A}_{1:T}$ is the action sequence taken in the independent run.
        \item The value function $B_t^V \defeq \E\left[ \left. \sum_{s=t}^T R_{A_s,s} \right| H_{t-1} \right]$.
\end{enumerate}

As proven in Theorem~\ref{thm:unbiasedness}, each of these baselines can be used in a combination with any of the reward metrics listed above.
The baseline $B_t^\text{oracle}$ is an instance-dependent measure that represents the performance of the omniscient policy that knows the values of true parameters $\theta$.
Given that $M_t^\text{mean}$ is chosen as a coupled reward metric, the gap $B_t^\text{oracle} - M_t^\text{mean}$ reduces to the ``regret'' which is a measure of suboptimality that has been widely used in bandit studies. This choice of baseline is natural when we expect adaptive algorithm to have small average regret. It can be less effective in problems with a short time horizon, where the reward earned by an oracle is not an attainable baseline.

The baseline $B_t^\text{self}$ utilizes an independent run of the same randomized policy under the same instance.
The idea of self-play was adopted from \citet{boutilier2020differentiable} while we make a generalization regarding the choice of reward metric and provide a formal proof of its validity.
It effectively centers the reward metric, i.e., $\E[ M_t - B_t^\text{self} ] = 0$, which helps
stabilize gradient estimates. In our numerical experiments, $B_t^\text{self}$ shows an impressive
performance across the different settings, though it effectively requires the computational effort
of running twice as many simulations.

Finally, the baseline $B_t^V$ is constructed analogously to the reward metric $M_t^Q$, and it represents the average performance of the given policy at the given state.
In a combination with $M_t^Q$, the gap $M_t^Q - B_t^V$ measures the relative benefit of the chosen action compared to the average, which is also known as the advantage function. Like $M_t^Q$, however, this baseline does not have a closed-form expression. The baseline $B_t^V$ can be understood as averaging the result of $B_t^\text{self}$ (applied with the posterior mean reward metric) across many independent runs of the algorithm. The randomized baseline $B_t^\text{self}$ has higher variance, but can be calculated at much lower computational cost.

\seclabel{Implementation issues.}
If we are equipped with a simulator that can generate instances with full information, it is straightforward to compute the reward metrics and the baselines listed above (apart from the computational efficiency).
If we are running the algorithm in the real world situation, however, we may not be able to identify their values as we do not have an access to unrevealed information such as true model parameters $\theta$.
Nevertheless, in the Bayesian setting, we can overcome this issue by sampling the unobserved variables at the end of an episode:
For example, after completing an episode, we can sample $\tilde{\theta} \sim \Pscr(Y_T)$ as if we perform one more step of \textsc{TS}, and plug them into the formulas for reward metric or baseline.
This is valid since $\tilde{\theta}$ is identically distributed with the true parameters $\theta$ given the observations revealed in that episode, by the virtue of posterior distribution, and therefore the resulting gradient estimates $\tilde{G}$ will also be identically distributed with the true one $G$.

Note that any model mis-specification can lead to a bias of the gradient estimator.
More specifically, if the prior distribution or the reward distribution is mis-specified (e.g., the value of noise variance $\sigma_a^2$ is incorrect in Example \ref{ex:gaussian-MAB}), the reward metrics $M_t^\text{fin}$ and $M_t^\text{Bayes}$ will result in biased estimates. If the mean reward function $\mu( \cdot, \cdot )$ is incorrect, furthermore, all the reward metrics other than $M_t^\text{obs}$ will suffer from the bias.
We expect that the users can determine whether there is an bias during the training process and adopt a more robust metric if needed.

\subsection{Variance Comparison}

The variance of a gradient estimator is a crucial factor for the performance of policy
gradient. In this section, we provide an analysis than can provide theoretical comparisons between
estimators of the form \eqref{eq:gradient-estimator-general}, including many of the estimators in
Section~\ref{sec:reward-metric-choices}.

To begin, note that for an admissible estimator of the form \eqref{eq:gradient-estimator-general},
we have
\[
  \E\left[ G \right] = \E\left[ G^{M,B} \right]
  = \E\left[ \sum_{t=1}^T S_t \cdot \left( M_t - B_t \right) \right]
  = T\times \E\left[ \frac{1}{T} \sum_{t=1}^T  G^{M,B}_t
  \right]
  = \E\left[ T \times  G^{M,B}_\tau \right],
\]
where $G^{M,B}_t \defeq S_t \cdot (M_t - B_t)$, and $\tau \in [T]$ is a random time index that is
independently and uniformly distributed. Thus, given any admissible estimator $G^{M,B}$, we can
construct a related \emph{single time} estimator $T \times G^{M,B}_\tau $ that is also
unbiased. Loosely speaking, this estimator estimates the gradient based on a the impact of an action
taken at a single, randomly chosen decision epoch $\tau$, rather than considering all decision
epochs. Moreover, the simpler, single time estimator is more amenable to analysis.

In the next theorem, we further provide a comparison between two single time
gradient estimators in terms of the variance they induce.  For two square symmetric matrices $A$
and $B$, we say $A \preceq B$ if and only if $B - A$ is a positive semi-definite matrix. This
gives a partial ordering of symmetric matrices.
\begin{theorem}[Variance reduction] \label{thm:variance-reduction}
  Consider two reward metric and baseline pairs $(\underline{M}_t, \underline{B}_t)$ and $(\overline{M}_t, \overline{B}_t)$ that satisfy
  \begin{equation} \label{eq:variance-reduction-relationship}
    \underline{M}_t - \underline{B}_t =
    \E\left[ \left. \overline{M}_t - \overline{B}_t \right| \Gscr_t \right]
    , \quad \forall t \in [T],
  \end{equation}
  for some $\Gscr_t \supseteq \sigma( H_{t-1}, \tilde{\theta}_t )$.
  Let $\underline{G}_\tau$ and $\overline{G}_\tau$ be corresponding single time gradient estimators,
  respectively, i.e.,
  \[
    \underline{G}_\tau \defeq S_\tau \cdot \big( \underline{M}_\tau - \underline{B}_\tau \big)
    , \quad
    \overline{G}_\tau \defeq S_\tau \cdot \big( \overline{M}_\tau - \overline{B}_\tau \big)
    .
  \]
  Then, $\overline{G}_\tau$ exhibits a smaller variance than $\underline{G}_\tau$, in the sense that
  \begin{equation} \label{eq:variance-reduction-cov}
    \Cov[ \underline{G}_\tau ] \preceq \Cov[ \overline{G}_\tau ].
  \end{equation}
\end{theorem}
\begin{proof}
  Fix $t \in [T]$. By the law of total covariance and conditioning on $\Gscr_t$,
  \begin{equation}\label{eq:cov-1}
    \begin{split}
      \Cov[ \overline{G}_t ]
      & = \Cov\left[ \E\big( \overline{G}_t \big| \Gscr_t \big) \right]
      + \E\left[ \Cov\big( \overline{G}_t \big|  \Gscr_t \big) \right]
      \\
      & \succeq \Cov\left[ \E\big( \overline{G}_t \big| \Gscr_t \big) \right]
      \\
      &= \Cov\left[ \mathbb{E}\left( \left. S_t \cdot \big( \overline{M}_t - \overline{B}_t
            \big) \right| \Gscr_t \right) \right]
      \\
      & = \Cov\left[ S_t \cdot \mathbb{E}\left( \left.
            \overline{M}_t - \overline{B}_t \right| \Gscr_t \right) \right]
      \\
      &= \Cov\left[ S_t \cdot (\underline{M}_t - \underline{B}_t) \right]
      \\
      & = \Cov[ \underline{G}_t ],
    \end{split}
  \end{equation}
  where the inequality in the second step follows from the fact that every covariance matrix is
  positive semi-definite.

  Now, note that
  \[
    \begin{split}
      \E\left[ \left. \overline{G}_\tau \right| \tau \right]
      & =
      \E\left[ \left.   S_\tau \cdot \E\left[ \left. \overline{M}_\tau - \overline{B}_\tau
              \right|
              \Gscr_\tau
            \right] \right| \tau \right]
      =
      \E\left[ \left.    S_\tau \cdot \left( \underline{M}_\tau - \underline{B}_\tau \right)
        \right| \tau \right]
      = \E\left[ \left. \underline{G}_\tau \right| \tau \right].
    \end{split}
  \]
  Then, applying the law of total covariance again, this time conditioning on
  $\tau$,
  \[
    \begin{split}
      \Cov[ \overline{G}_\tau ]
      & = \Cov\left[ \E\big( \overline{G}_\tau \big| \tau \big) \right]
      + \E\left[ \Cov\big( \overline{G}_\tau \big|  \tau \big) \right]
      \\
      & = \Cov\left[ \E\big( \underline{G}_\tau \big| \tau \big) \right]
      + \E\left[ \Cov\big( \overline{G}_\tau \big|  \tau \big) \right]
      \\
      & \succeq  \Cov\left[ \E\big( \underline{G}_\tau \big| \tau \big) \right]
      + \E\left[ \Cov\big( \underline{G}_\tau \big|  \tau \big) \right]
      \\
      &  = \Cov[ \underline{G}_\tau ].
    \end{split}
  \]
  Here, the inequality follows from \eqref{eq:cov-1}.
  This concludes the proof.
\end{proof}

Theorem~\ref{thm:variance-reduction} provides a pairwise comparison of two single time gradient
estimators (i.e., $\overline{G}_\tau$ and $\underline{G}_\tau$), when their reward metrics and baselines
are related by \eqref{eq:variance-reduction-relationship}.  Ideally we would like a comparison
between the variance of the original gradient estimators (i.e.,
$\Cov\big[ \sum_{t=1}^T \underline{G}_t \big]$ and
$\Cov\big[ \sum_{t=1}^T \overline{G}_t \big]$). However, this is challenging due to the
interdependence across time between the score functions and the reward metrics. Nevertheless, we
believe that Theorem~\ref{thm:variance-reduction} is informative, and the ordering it implies is
consistent with the numerical performance results we will see in Section~\ref{sec:numeric}.

Theorem~\ref{thm:variance-reduction} implies that the reward metric based on the smaller
information set (i.e., through more averaging) produces a more precise gradient estimator than one
based on the larger information set (i.e., with less averaging). This is the same insight that
drives the Rao-Blackwell theorem.

In the development of the reward metrics in Section~\ref{sec:reward-metric-choices}, we have
argued that some reward metrics are motivated from the others via Rao-Blackwellization.  In fact,
the relationship \eqref{eq:variance-reduction-relationship} holds among the reward metrics
$M_t^\text{obs}$, $M_t^\text{mean}$, $M_t^\text{fin}$, and $M_t^Q$ (not including
$M_t^\text{Bayes}$). Indeed, an application of Theorem~\ref{thm:variance-reduction} immediately
yields the following ordering among the gradient estimators:
\begin{corollary}
\[
  \Cov[ G_\tau^{M^\textup{obs},B} ] \succeq \Cov[ G_\tau^{M^\textup{mean},B} ] \succeq
  \Cov[ G_\tau^{M^\textup{fin},B} ] \succeq \Cov[ G_\tau^{M^Q,B} ],
\]
for any choice of baseline $B$ from $B^\textup{null}$, $B^\textup{oracle}$, $B^\textup{self}$, and
$B^V$. Here, note that the baselines $B_t^\textup{oracle}$ and $B_t^\textup{self}$ require a coupled
reward metric. We assume they are coupled to the corresponding reward metric in use in each
estimator.
\end{corollary}


\section{Numerical Experiments} \label{sec:numeric}

In this section, we report the simulation results of the policy gradient optimization of Thompson
sampling.  We aim to illustrate the flexibility of our proposed framework as a meta-learning
platform for bandit tasks, compare the gradient estimators with different choices of reward metric
and baseline, and highlight the performance of optimized sampling policies in a comparison with the
other state-of-the-art algorithms.

\seclabel{Setup.}
We consider Gaussian multi-armed bandit (MAB) problems, introduced in Example \ref{ex:gaussian-MAB}, for which we implement \textsc{TS} with parameterized posterior reshaping, described in Example \ref{ex:posterior-reshaping-Gaussian}.
To highlight the improvement over the na\"ive \textsc{TS}, our experiments include the following configurations:
\begin{enumerate}
        \item Gaussian MAB with 10 arms ($K=10$) and 500 time periods ($T=500$), where all arms have the same prior distribution and the same noise variance.
                This is a typical setting that has been studied in many prior works.
        \item Gaussian MAB with heteroscedastic reward distributions, where we are given 5 arms ($K=5$) with very different noise variances and 50 time periods ($T=50$).
                Since each arm requires a different amount of effort to learn its unknown mean, it is important to incorporate information about the noise variances into the decision making, which standard \textsc{TS} does not do.
        \item Gaussian MAB with an excessive number of arms, where we are given 20 arms ($K=20$) with identical priors and 20 time periods ($T=20$). In this setup, there is no hope of discovering the true optimal arm. Nevertheless, standard \textsc{TS} continues to select arms that have never been tried throughout the entire time horizon, which is very wasteful exploration.
\end{enumerate}
Note that in all of these settings we have adopted the same parameterization of \textsc{TS}.
This is to verify that our proposed framework achieves the goal of meta-learning:
The policy gradient procedure finds the choice of meta-parameters $\lambda \in \mathbb{R}^{4K}$
from Example~\ref{ex:posterior-reshaping-Gaussian} that is optimized for each of the bandit
settings, resulting in the algorithm that is trained to exploit the structure in each setting and
performs no worse than the standard version of \textsc{TS}. We highlighted that the optimized
behavior differs substantially across settings and at times differs substantially from
\textsc{TS}.

\seclabel{Training.}  We implement the policy gradient algorithm based on the gradient estimator
\eqref{eq:gradient-estimator-general} with different combinations of reward metric $M$ and
baseline $B$.  In each policy gradient iteration, we compute the batch gradient, i.e., the average
gradient measured across a set of independently generated bandit instances, where the batch size
(the number of instances) ranges from 1,000 to 5,000 across the settings. The Adam optimizer
\citep{kingma2014adam} is used to perform the gradient ascent steps.

The random generation of Gaussian MAB instances is done according to the model described in
Example \ref{ex:gaussian-MAB}.  To facilitate an accurate comparison between the gradient
estimators, the estimators share all the randomness in the instance generation and the random
action selection. That is, in the notation of \eqref{eq:explicit-randomness}, the same
realizations of noise variables $(\epsilon_t)$ and $(\xi_t)$ are used for the simulation of
different policy gradient estimators.

\seclabel{Evaluation.}
As a suboptimality measure of a bandit algorithm, we utilize the (Bayesian) regret defined as follows:
\begin{equation} \label{eq:regret}
        \textsc{Regret}(\pi) \defeq \E\left[ \sum_{t=1}^T \max_{a \in \Ascr} \{ \mu( \theta, a ) \} - \mu( \theta, A_t ) \right],
\end{equation}
which is measured via sample average approximation in our simulation.
When computing the gradient estimator during the training process, we obtain as a side product the regret that the algorithm incurs in each training batch, and we report this trajectory of regret as a learning curve of the policy gradient optimization.
We naturally expect that the regret decreases as training proceeds.
Finally, we measure the regret of the trained policies (and the other bandit algorithms listed below) on the evaluation batch, which is a set of instances generated independently of the training batches.
As done in training, the same set of instances are used for evaluating all the policies so as to facilitate accurate comparisons among them.

\seclabel{Competing bandit algorithms.}  We consider the state-of-the-art bandit algorithms that
are suitable for a Bayesian setting: the Bayesian upper confidence bound
\citep{kaufmann2012bayesian} (\textsc{Bayes-UCB}, with a quantile of $1 - \tfrac{1}{t}$),
information-directed sampling \citep{russo2018learning} (\textsc{IDS}), and the optimistic Gittins
index\footnote{There are two free parameters in \textsc{OGI}. We use a one-step look-ahead and a
  discount factor of $\gamma_t = 1 - \tfrac{1}{t}$, which was the primary focus of
  \citet{gutin2016optimistic}.} \citep{gutin2016optimistic} (OGI). We compare the performance of
the trained \textsc{TS} policies with these algorithms.

\seclabel{Implementation.}
All the code is written in Python, and the training module is implemented using Tensorflow to utilize the automatic gradient calculation and the Adam optimizer.
We use 64-bits floating numbers for computation of gradient estimator.

\subsection{Gaussian MAB in a Standard Setting ($K=10, T=500$)}
We first report the result for Gaussian MAB with 10 arms and 500 time periods.
More specifically, we are given ten independent arms with identical prior distributions: For each arm $a = 1, \ldots, K$ and time $t = 1, \ldots, T$, we assume that
\begin{equation}
        \theta_a \sim \Nscr( 0, 1^2 )
        , \quad
        R_{a,t} | \theta_a \sim \Nscr( \theta_a, 1^2 ).
\end{equation}
This setup has been also examined in the prior literature \citep{gutin2016optimistic,
  russo2018learning}.

For policy gradient optimization of \textsc{TS} with parameterized posterior reshaping, we adopt the various combinations of reward metric $M$ and baseline $B$ for the gradient estimator $G^{M,B}$.
The initial values for the meta-parameters $\lambda$ are chosen in the way that the corresponding policy is identical to the standard \textsc{TS}.
The training batch size is set to 5,000 and the learning rate for Adam optimizer is set to 0.01.

Figure \ref{fig:numeric-gauss-long} shows the learning curves obtained in our simulated training, and Table \ref{tab:numeric-gauss-long} reports the performance of the trained \textsc{TS} policies as well as the other algorithms being compared.
In every combination of reward metric and baseline, we observe a steady improvement in performance over the course of the training process (starting from the standard \textsc{TS}).
The training performance largely depends on the choice of baseline: with baseline $B^\text{oracle}$ or $B^\text{self}$ the algorithm shows an impressive progress, catching the state-of-the-art algorithms within 300 policy gradient iterations and ending up with policies that improve over the standard \textsc{TS} by 23\% in terms of regret.

\begin{figure}[H]
        \centering
        \includegraphics[width=0.8\linewidth]{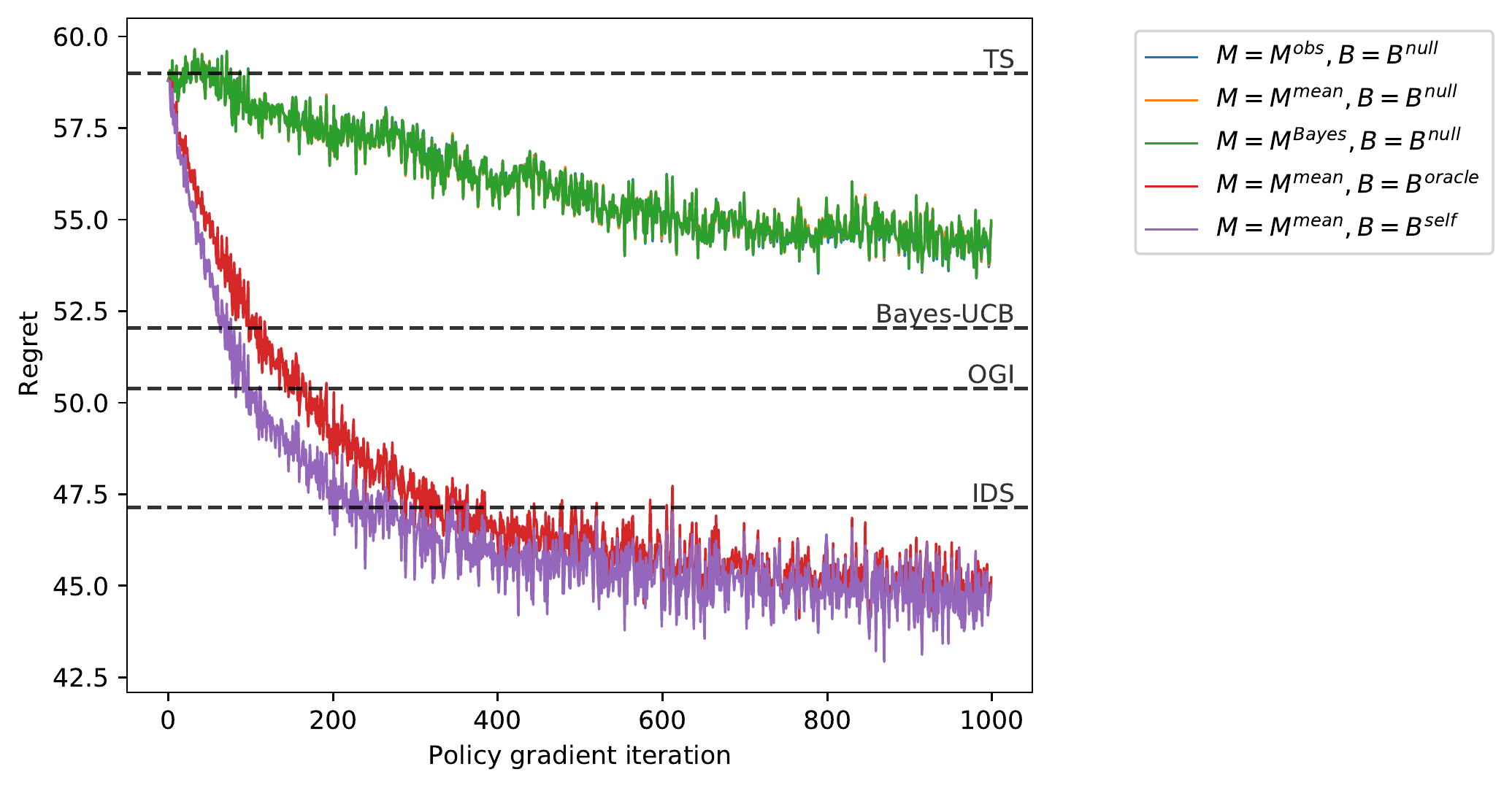}
        \caption{
                Learning curves of parameterized Thompson sampling trained for Gaussian MAB with 10 arms and 500 time periods.
                A curve shows the progress of policy gradient optimization based on a particular choice of reward metric $M$ and baseline $B$ for the gradient estimator $G^{M,B}$, defined in \eqref{eq:gradient-estimator-general}.
                The $k^\text{th}$ data point in each curve reports the average regret on the $k^\text{th}$ training batch, which contains 5,000 independent instances.
                The dashed horizontal lines represent the performance of the other algorithms measured with the evaluation batch containing 20,000 instances (also reported in Table \ref{tab:numeric-gauss-long}).
                In this plot, the learning curves of $(M^\text{obs}, B^\text{null})$, $(M^\text{mean}, B^\text{null})$, and $(M^\text{Bayes}, B^\text{null})$ precisely coincide.
        }
        \label{fig:numeric-gauss-long}
\end{figure}

        \begin{table}[H]
        \centering
        \begin{tabular}{*4c}
        \toprule
        \thead{Algorithm} & \thead{Reward metric} & \thead{Baseline} & \thead{Regret (s.e.)} \\
        \midrule
        \multirow{5}{*}{ Trained \textsc{TS} }
                & $M^\text{obs}$ & $B^\text{null}$ & 54.411 (0.206) \\
                & $M^\text{mean}$ & $B^\text{null}$ & 54.530 (0.203) \\
                & $M^\text{Bayes}$ & $B^\text{null}$ & 54.591 (0.208) \\
                & $M^\text{mean}$ & $B^\text{oracle}$  & 45.247 (0.302) \\
                & $M^\text{mean}$ & $B^\text{self}$ & \textbf{45.099} (0.320) \\
         \midrule
         Na\"ive \textsc{TS} & -- & -- & 58.999 (0.191) \\
         \textsc{Bayes-UCB} & -- & -- & 52.038 (0.186) \\
         \textsc{OGI} & -- & -- & 50.381 (0.348) \\
         \textsc{IDS} & -- & -- & 47.135 (0.335) \\
        \bottomrule
        \end{tabular}
        \caption{
                Performance of the algorithms for Gaussian MAB with 10 arms and 500 time periods.
                Each trained \textsc{TS} uses the meta-parameters that are obtained from training procedure, i.e., the ones found at the end of 1,000 iterations of batched policy gradient ascent (Figure \ref{fig:numeric-gauss-long}).
                The performance is measured in regret, defined in \eqref{eq:regret}, and computed via sample average approximation over 20,000 independent instances, and reported with the standard error.
                The best result is emphasized with bold letters.
                }
        \label{tab:numeric-gauss-long}
        \end{table}

\subsection{Gaussian MAB with Heteroscedastic Arms ($K=5, T=50$)}
We now explore a different configuration of the Gaussian MAB under which the na\"ive \textsc{TS} performs particularly poorly.
We consider five arms that have significantly different noise variances: For each arm $a=1,\ldots,5$ and time $t=1,\ldots,50$, we assume that
\begin{equation}
        \theta_a \sim \Nscr( 0, 1^2 ), \quad
        R_{a,t} | \theta_a \sim \Nscr( \theta_a, \sigma_a^2 ), \quad
        \text{where} \quad
        \sigma_{1:5}^2 \defeq ( 0.1,~ 0.4,~ 1,~ 4,~ 10 ).
\end{equation}
Note that it is crucial for the algorithms to consider the heterogeneity in the reward variances since the variance $\sigma_a^2$ determines how much the decision maker can learn about the unknown mean reward $\theta_a$ within a finite number of observations: in order for the posterior distribution to concentrate so as to have the standard deviation of 0.1, for example, a single observations is enough for arm 1 whereas 100 and 10,000 observations are required for arm 3 and arm 5, respectively. This is especially important when the time horizon is short, as in this case.

We use the training batches of size 1,000 for gradient estimation, and the Adam optimizer with learning rate of 0.05 for policy gradient, and the evaluation batch of size 10,000 for evaluation. While every combination of reward metric and baseline shows a very stable progress throughout the policy gradient procedure, as shown in Figure \ref{fig:numeric-gauss-hetero}, we observe that the baseline $B^\text{self}$ works slightly better than the baseline $B^\text{oracle}$, and so does $B^\text{oracle}$ than $B^\text{null}$.

The evaluation results are shown in Table \ref{tab:numeric-gauss-hetero}.
We immediately observe that na\"ive \textsc{TS} and \textsc{Bayes-UCB} particularly perform poorly as they make decisions based only on the posterior at each moment without incorporating the noise variances into consideration.
As the results show, by optimizing posterior reshaping parameters, we can make \textsc{TS} to trade off exploitation and exploration much more precisely, so that we can achieve a surprising improvement over \textsc{TS} by 35\%--40\%.

\begin{figure}[H]
        \centering
        \includegraphics[width=0.8\linewidth]{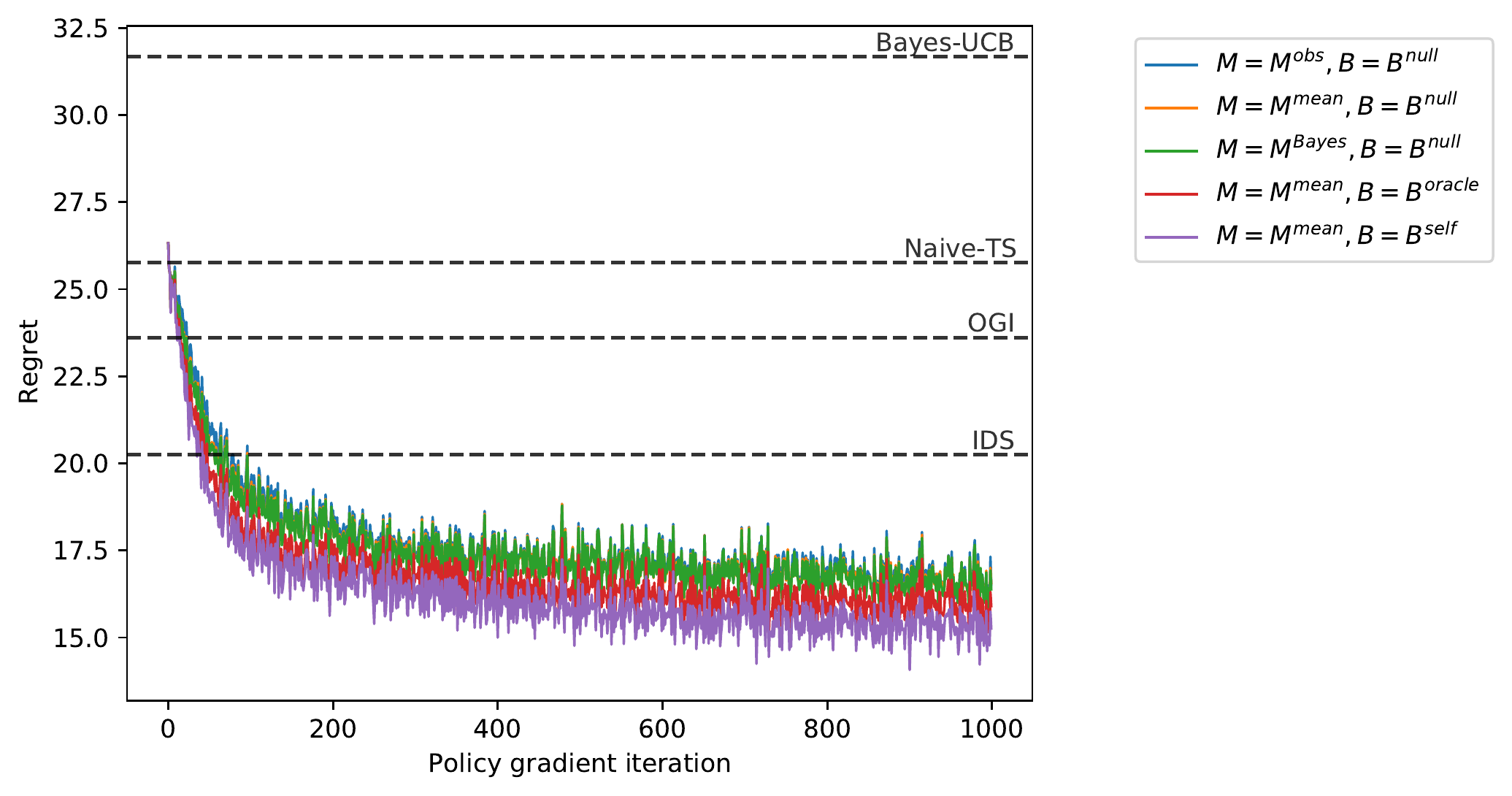}
        \caption{
                Learning curves of parameterized Thompson sampling trained for Gaussian MAB with heteroscedastic arms ($K=5, T=50, \sigma_{1:5}^2 = ( 0.1, 0.4, 1, 4, 10 )$).
                A curve shows the progress of policy gradient optimization based on a particular choice of reward metric $M$ and baseline $B$ for the gradient estimator $G^{M,B}$, defined in \eqref{eq:gradient-estimator-general}.
                The $k^\text{th}$ data point in each curve reports the average regret on the $k^\text{th}$ training batch, which contains 1,000 independent instances.
                The dashed horizontal lines represent the performance of the other algorithms measured with the evaluation batch containing 10,000 instances (also reported in Table \ref{tab:numeric-gauss-hetero}).
                In this plot, the learning curves of $(M^\text{obs}, B^\text{null})$, $(M^\text{mean}, B^\text{null})$, and $(M^\text{Bayes}, B^\text{null})$ precisely coincide.
        }
        \label{fig:numeric-gauss-hetero}
\end{figure}

        \begin{table}[H]
        \centering
        \begin{tabular}{*4c}
        \toprule
        \thead{Algorithm} & \thead{Reward metric} & \thead{Baseline} & \thead{Regret (s.e.)} \\
        \midrule
        \multirow{5}{*}{ Trained \textsc{TS} }
                & $M^\text{obs}$ & $B^\text{null}$ & 16.561 (0.204) \\
                & $M^\text{mean}$ & $B^\text{null}$ & 16.582 (0.202) \\
                & $M^\text{Bayes}$ & $B^\text{null}$ & 16.492 (0.202) \\
                & $M^\text{mean}$ & $B^\text{oracle}$  & 15.959 (0.203) \\
                & $M^\text{mean}$ & $B^\text{self}$ & \textbf{15.310} (0.198) \\
         \midrule
         Na\"ive \textsc{TS} & -- & -- & 25.768 (0.156) \\
         \textsc{Bayes-UCB} & -- & -- & 31.677 (0.254) \\
         \textsc{OGI} & -- & -- & 23.614 (0.224) \\
         \textsc{IDS} & -- & -- & 20.249 (0.202) \\
        \bottomrule
        \end{tabular}
        \caption{
                Performance of the algorithms for Gaussian MAB with heteroscedastic arms.
                Each trained \textsc{TS} uses the meta-parameters found at the end of 1,000 iterations of batched policy gradient ascent (Figure \ref{fig:numeric-gauss-hetero}).
                The performance is measured in regret, defined in \eqref{eq:regret}, and computed via sample average approximation over 10,000 independent instances, and reported with the standard error.
                The best result is emphasized with bold letters.
                }
        \label{tab:numeric-gauss-hetero}
        \end{table}

\subsection{Gaussian MAB with an Excessive Number of Arms ($K=20, T=20$)}\label{subsec: many arms}
We finally investigate Gaussian MAB with an excessive number of arms, i.e., too many arms compared to the length of time horizon.
More specifically, we consider 20 arms and 20 time periods: For each arm $a = 1, \ldots, 20$ and time $t = 1, \ldots, 20$, we assume that
\begin{equation}
        \theta_a \sim \Nscr( 0, 1^2 )
        , \quad
        R_{a,t} | \theta_a \sim \Nscr( \theta_a, 1^2 ).
\end{equation}

This setup is motivated from \citet{russo2018satisficing} in which the authors posit an extreme example where \textsc{TS} faces an infinite number of arms with identical priors.
In such an example, \textsc{TS} keeps pulling a new arm throughout the entire process, since with zero probability the same arm gets the largest sampled mean $\tilde{\theta}_{a,t}$ more than once: As a result, \textsc{TS} does not utilize any information obtained from the past pulls, and always earns the prior mean $\mathbb{E}[ \theta_a ]$ in expectation at each time.
We aim to see whether \textsc{TS} can resolve this over-exploration issue if optimized via policy gradient.

As in the previous setup, we use the training batches of size 1,000 and the learning rate of 0.05 for Adam optimizer, and the evaluation batch of size 10,000 for evaluation.

The simulation results are reported in Figure \ref{fig:numeric-gauss-many} and Table \ref{tab:numeric-gauss-many}.
As expected, the na\"ive \textsc{TS} exhibits an extremely poor performance in this setup.
At the end of the training process, we observe that all trained algorithms have almost identical performance regardless of the choice of reward metric and baseline in their gradient estimation.
During the initial phase of training (i.e., during the first 300 iterations), in contrast, we can observe that the baseline $B^\text{null}$ performs better than the baseline $B^\text{oracle}$.
This is in contrast with the heteroscedastic noise example. The intuitive reason is that, in the current setup, the oracle performance does not provide a (nearly) attainable benchmark for an adaptive algorithm.   

Figure \ref{fig:numeric-gauss-many-count} visualizes how the parameterized \textsc{TS} is improved over the course of training. It shows the distribution of pulls that each arm gets, measured at the beginning, middle, and end of the training.
At the beginning, since the initial values for meta-parameters are chosen to yield the standard version of \textsc{TS}, it allocates the pulls evenly across the arms, i.e., one pull per one arm in average.
As training proceeds, we can observe that the distribution becomes more skewed, i.e., the algorithm effectively ignored some arms as it realizes that it is wasteful to explore all of the arms.
The set of ignored arms are randomly determined during the course of policy gradient optimization: While not reported here, across the choices of reward metric and baseline, the shape of the distribution looks alike but the ordering of arms in the distribution is observed to be different.

\begin{figure}[H]
        \centering
        \includegraphics[width=0.8\linewidth]{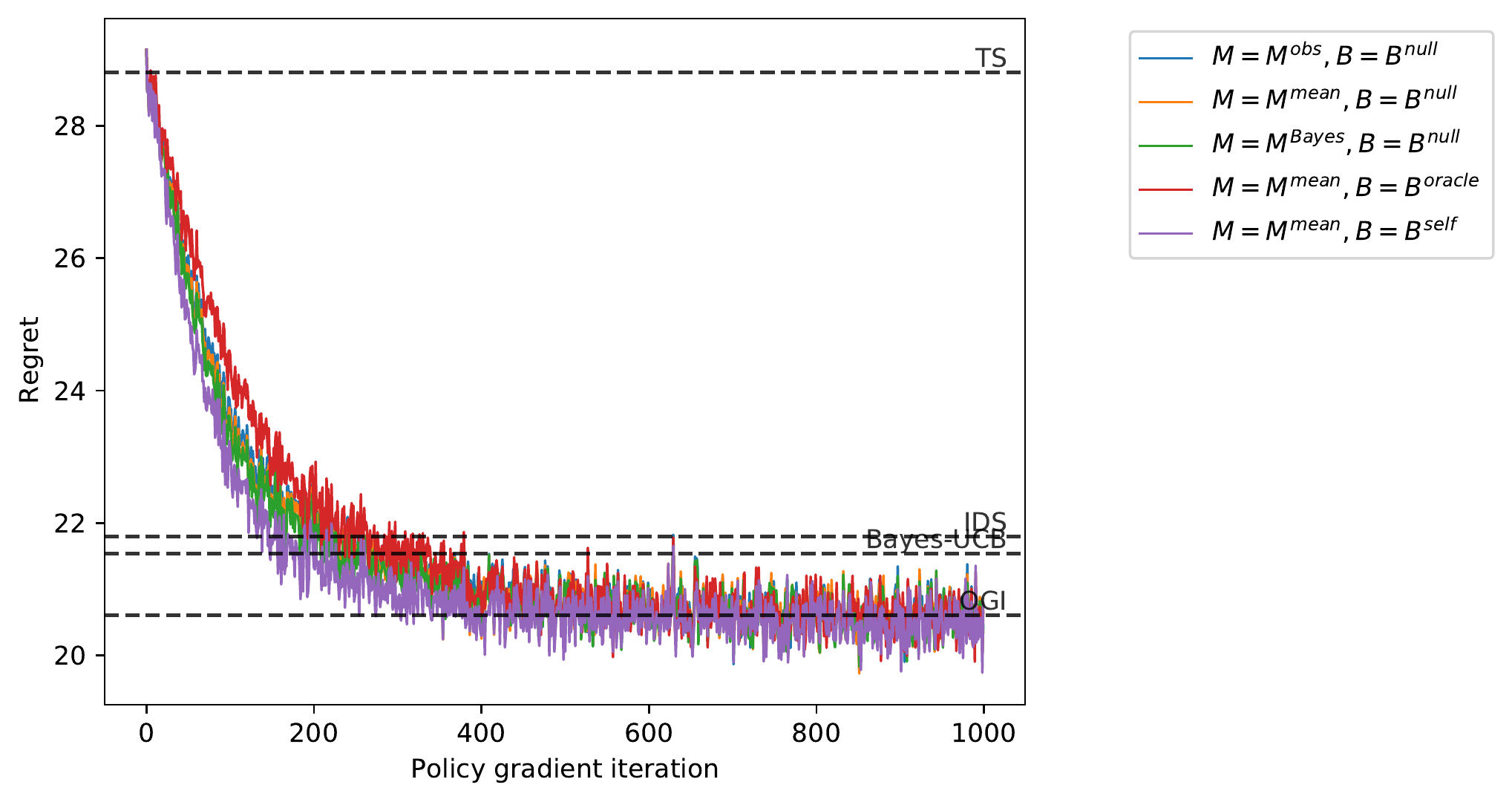}
        \caption{
                Learning curves of parameterized Thompson sampling trained for Gaussian MAB with an excessive number of arms ($K=20, T=20$).
                A curve shows the progress of policy gradient optimization based on a particular choice of reward metric $M$ and baseline $B$ for the gradient estimator $G^{M,B}$, defined in \eqref{eq:gradient-estimator-general}.
                The $k^\text{th}$ data point in each curve reports the average regret on the $k^\text{th}$ training batch, which contains 1,000 independent instances.
                The horizontal lines represent the performance of the other algorithms measured with the evaluation batch containing 10,000 instances (also reported in Table \ref{tab:numeric-gauss-many}).
                In this plot, the learning curves of $(M^\text{obs}, B^\text{null})$, $(M^\text{mean}, B^\text{null})$, and $(M^\text{Bayes}, B^\text{null})$ precisely coincide.
        }
        \label{fig:numeric-gauss-many}
\end{figure}

        \begin{table}[H]
        \centering
        \begin{tabular}{*4c}
        \toprule
        \thead{Algorithm} & \thead{Reward metric} & \thead{Baseline} & \thead{Regret (s.e.)} \\
        \midrule
        \multirow{5}{*}{ Trained \textsc{TS} }
                & $M^\text{obs}$ & $B^\text{null}$ & 20.625 (0.127) \\
                & $M^\text{mean}$ & $B^\text{null}$ & 20.541 (0.126) \\
                & $M^\text{Bayes}$ & $B^\text{null}$ & 20.591 (0.127) \\
                & $M^\text{mean}$ & $B^\text{oracle}$  & \textbf{20.348} (0.126) \\
                & $M^\text{mean}$ & $B^\text{self}$ & 20.421 (0.128) \\
         \midrule
         Na\"ive \textsc{TS} & -- & -- & 28.802 (0.097) \\
         \textsc{Bayes-UCB} & -- & -- & 21.537 (0.124) \\
         \textsc{OGI} & -- & -- & 20.604 (0.126) \\
         \textsc{IDS} & -- & -- & 21.799 (0.118) \\
        \bottomrule
        \end{tabular}
        \caption{
                Performance of the algorithms for Gaussian MAB with an excessive number of arms.
                Each trained \textsc{TS} uses the meta-parameters found at the end of 1,000 iterations of batched policy gradient ascent (Figure \ref{fig:numeric-gauss-many}).
                The performance is measured in regret, defined in \eqref{eq:regret}, and computed via sample average approximation over 10,000 independent instances, and reported with the standard error.
                The best result is emphasized with bold letters.
                }
        \label{tab:numeric-gauss-many}
        \end{table}

\begin{figure}[H]
        \centering
        \includegraphics[width=0.8\linewidth]{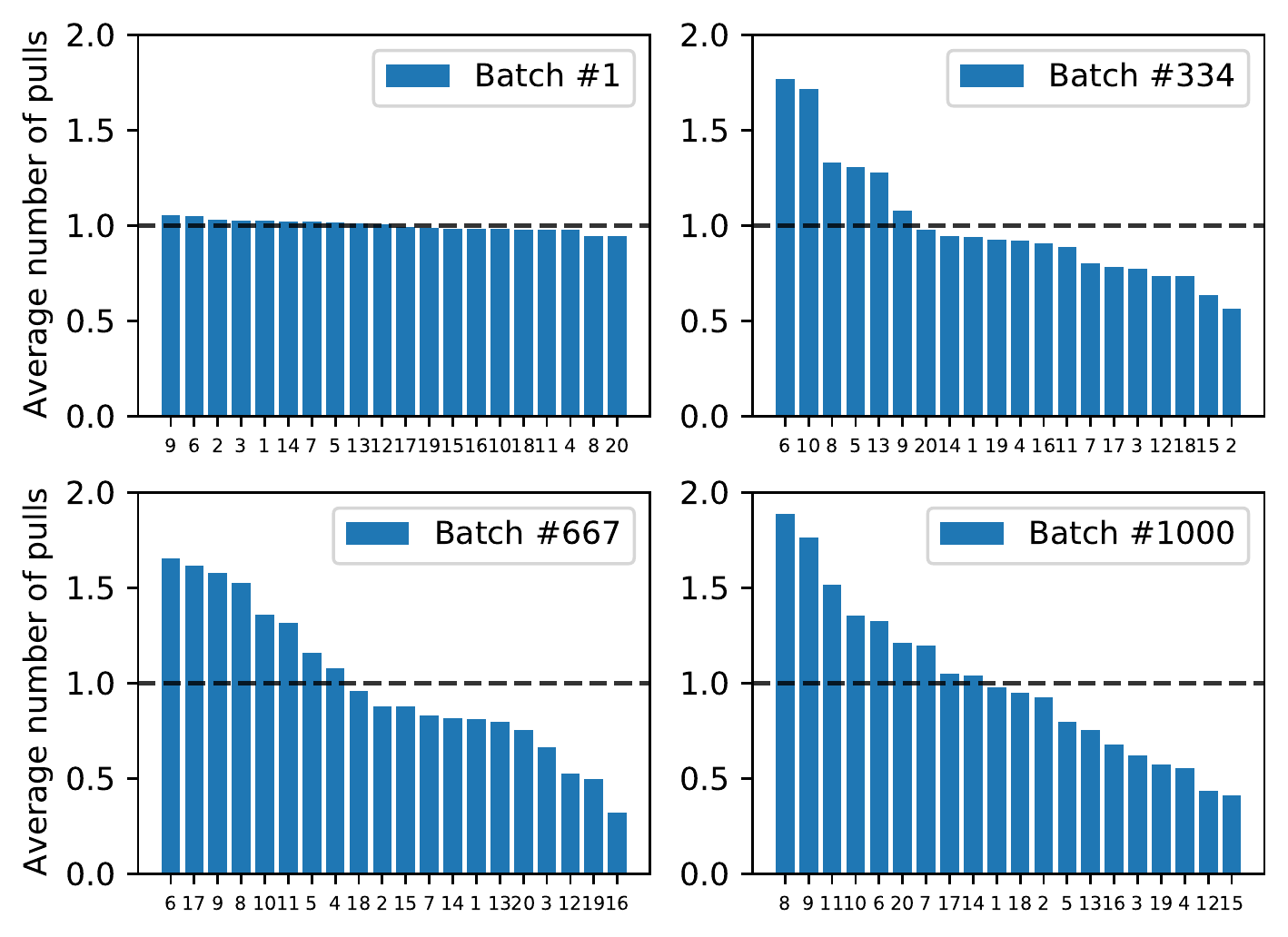}
        \caption{
                Average number of pulls that the parameterized \textsc{TS} conducts on each arm in the $k^\text{th}$ training batch during the course of policy gradient optimization with reward metric $M^\text{mean}$ and baseline $B^\text{self}$, where $k \in \{1, 334, 667, 1000\}$.
                Each training batch contains 1,000 independent instances, and the horizontal axes show the arm indices rearranged in the order of the average number of pulls.
                }
        \label{fig:numeric-gauss-many-count}
\end{figure}



\bibliography{references}

\newpage

\appendix

\newpage

\end{document}